\documentclass[journal,fleqn,10pt]{IEEEtran}
\usepackage{cite}
   \usepackage[pdftex]{graphicx}
 \usepackage[cmex10]{amsmath}
 \usepackage{algorithmic}
 \usepackage{array}
 \usepackage{url}

\usepackage[english]{babel}

\usepackage{amsmath,amsfonts,amssymb,amsbsy}
\usepackage{amsthm,mathrsfs}
\usepackage{graphicx,color}
\usepackage{verbatim}
\usepackage{pslatex}
\usepackage{cite,url,bm}
\usepackage{algorithm}
\usepackage{algorithmic}
\usepackage{epsf,subfigure,psfig}
\usepackage{latexsym,amsmath,epsfig}
\usepackage{multirow}
\renewcommand{\algorithmicrequire}{\textbf{Input:}}
\renewcommand{\algorithmicensure}{\textbf{Output:}}

\newcommand{\vect}[1]{\pmb{#1}}
\newcommand{\mat}[1]{\pmb{#1}}

\newtheorem{prop}{Proposition}
\newtheorem{theorem}{Theorem}
\newtheorem{lemma}{Lemma}

\def\ben{\begin{equation*}}
\def\een{\end{equation*}}
\def\be{\begin{equation}}
\def\ee{\end{equation}}
\def\beaa{\begin{eqnarray*}}
\def\eeaa{\end{eqnarray*}}
\def\bea{\begin{eqnarray}}
\def\eea{\end{eqnarray}}

\def\bleq{\begin{flalign}}
\def\eleq{\end{flalign}}

\begin{document}
%
\title{ICR: Iterative Convex Refinement for Sparse Signal Recovery Using Spike and Slab Priors}
%
%
%

\author{Hojjat S. Mousavi, \emph{Student Member, IEEE}, Vishal Monga, \emph{Senior Member, IEEE}, and Trac D. Tran, \emph{Fellow, IEEE}
\thanks{H. S. Mousavi and V. Monga are with the Department
of Electrical  Engineering, The Pennsylvania State University, University Park,
PA, 16802 USA, e-mail: hojjat@psu.edu.}
\thanks{T. D. Tran is with the Department of Electrical and Computer Engineering, Johns Hopkins University, Baltimore, MD, 21218 USA.}
\thanks{This work has been supported partially by  the Office of Naval Research (ONR) under Grant N00014-12-1-0765 and NSF CAREER award to (V.M.)}
}



\maketitle

\begin{abstract}
In this letter, we address sparse signal recovery using spike and slab priors. In particular, we focus on a Bayesian framework where sparsity is enforced on reconstruction coefficients via probabilistic priors. The optimization resulting from spike and slab prior maximization is known to be a hard non-convex problem, and existing solutions involve simplifying assumptions and/or relaxations. We propose an approach called Iterative Convex Refinement (ICR) that aims to solve the aforementioned optimization problem directly allowing for greater generality in the sparse structure. Essentially, ICR solves a sequence of convex optimization problems such that sequence of solutions converges to a sub-optimal solution of the original hard optimization problem. We propose two versions of our algorithm: a.) an unconstrained version, and b.) with a non-negativity constraint on sparse coefficients, which may be required in some real-world problems. Experimental validation is performed on both synthetic data and for a real-world image recovery problem, which illustrates merits of ICR over state of the art alternatives.
\end{abstract}


%
\IEEEpeerreviewmaketitle

\section{Introduction}

\IEEEPARstart{S}{P}arse signal approximation and compressive sensing (CS) have recently gained considerable interest both in signal and image processing as well as statistics. Sparsity is  often a natural assumption in inverse problems and sparse reconstruction or representation has variety of applications in image/signal classification \cite{Wright:SRC_PAMI2009,Srinivas:SHIRC_TMI2014,Srinivas:SSPIC_ICIP2013, Srinivas:SHIRC_ISBI2013, Mousavi:MICHS_ICIP2014, Bahrampour:TreeSparsity_CVPR2014}, dictionary learning \cite{Suo1:DirtyDicLearn_ICIP2014, Pourkamali:CompresiveKSVD_ICASSP2013, SadeghiAndBabaiezade1:DicLearnSparse_SPLetter2013, Vu:DFDL_ISBI2015, Bahrampour:DicLearn_Arxiv2015, Bahrampour:KernelDicLearn_Arxiv2015}, signal recovery \cite{Wright:SpaRSA_TSP2009 , Tropp:OMP_InfoTheory2007}, image denoising and inpainting \cite{Elad:ImageDenoiseSparsity_TIP2006 }, super resolution \cite{YangAndWright:SparseSR_TIP2010 } and  MRI image reconstruction \cite{ Andersen:BayesianSpikeSlab_NIPS2014}. Typically, sparse models assume that a signal can be efficiently represented as sparse linear combination of atoms in a given or learned dictionary \cite{Wright:SRC_PAMI2009,Sprechmann:CHI-LASSO_TSP2011}. In other words, from CS viewpoint, a sparse signal can be recovered from fewer number of observations \cite{Baraniuk:Model_CS_InfoTheory2010, Carin:WaveletBayesCS_TSP2009, JiAndCarin:BayesianCS_TSP2008 }.

A typical sparse reconstruction algorithm aims to recover a sparse signal $\vect x \in \mathbb{R}^{p }$ from a set of fewer measurements $\vect y \in \mathbb{R}^{q}$ ($q\ll p$) according to the following model:
\bea
    \vect y = \mat A \vect x + \vect n,     \label{Eq:y=Ax}
\eea
where $\mat A \in \mathbb{R}^{q\times p}$ is the measurement matrix (Dictionary) and $\vect n \in \mathbb{R}^{q}$ models the additive Gaussian noise with variance $\sigma^2$.

In recent years, many sparse recovery algorithms have been proposed including but not limited to the following: proposing sparsity promoting optimization problems involving different regularizers such as $\ell_1$ norm, $\ell_0$ pseudo norm, greedy algorithms \cite{Tropp:OMP_InfoTheory2007,  Cai:OMP_InfoTheory2011, Mousavi:AssymLASSO_arXive2013, Mohimani:fast_l_0_TSP2009}, Bayesian-based methods \cite{JiAndCarin:BayesianCS_TSP2008, Lu:SparseCodeBayesPerspec_NeuralNetLearn2013, DobigeonAndHero:HierarchyBayesImageRecons_TIP2009} or general sparse approximation algorithms such as SpaRSA, ADMM, etc. \cite{Wright:SpaRSA_TSP2009, BeckerAndCandes:SparseRecoveryNESTA_ImagScienSIAM2011, Beck:IterativeShrinkageThresholdFISTA_ImagScienSIAM2009, Boyd:ADMM_MachineLearn2011}.

In this letter, we focus on sparse recovery from a Bayesian perspective by using hierarchical priors. In Bayesian sparse recovery, the choice of priors plays a key role in promoting sparsity and improving performance. Examples of such priors are Laplacian \cite{Babacan_BayesianCSLaplacePriors_TIP2010}, generalized Pareto \cite{Cevher:SparseRecovGraphicalModel_SPMagaz2010}, Spike and Slab \cite{Mitchell:BayesVarSelectSpikeSlab_StatAssoc1988}, etc. Amongst these priors, a well-suited sparsity promoting prior is spike and slab prior which is widely used in sparse recovery and Bayesian inference for variable selection and regression \cite{Ishwaran_SpikeSlab_AnnStat2005, Carin:WaveletBayesCS_TSP2009, Andersen:BayesianSpikeSlab_NIPS2014, Suo:HierarchySpikeSlab_ICASSP2013}. In fact, it is acknowledged that spike and slab prior is indeed the \emph{gold standard} for inducing sparsity in Bayesian inference \cite{Lazaro:SpikeSlabInferMultiTask_NIPS2011}.


Using these priors for sparse recovery leads to non-convex, non-smooth, mixed integer programming optimization problems which are often solved by means of relaxation or simplifying assumptions on the model parameters \cite{Yen:MM_VariableSelectionSpikeSlab_Stat2011, Srinivas:SSPIC_ICIP2013, Andersen:BayesianSpikeSlab_NIPS2014}. However, in this work we aim to solve the spike and slab optimization problem directly in its general form. Motivated by this, the \textbf{Main Contributions} of our work are as follows: (1) We propose a novel Iterative Convex Refinement (ICR) method to solve the optimization problem resulting from exploiting spike and slab priors. Essentially,  the sequence of solutions from these convex problems approaches a sub-optimal solution of the hard non-convex problem. (2) We propose two versions of ICR: a.) an unconstrained version, and b.) with a non-negativity constraint on sparse coefficients, which may be required in some real-world problems such as image recovery. (3) Finally, we perform experimental validation on both synthetic data and a realistic image recovery problem, which reveals the benefits of ICR over other state-of-the-art recovery methods using spike and slab priors. Further, we compare the solution of various sparse recovery methods against the \emph{global solution} for a small-scale problem, and remarkably the proposed ICR finds the most agreement with the global solution. Finally, convergence analysis is provided in support of the proposed ICR algorithm.
\section{Spike and Slab Sparse Signal Recovery }
Introducing priors for capturing sparsity is a particular example of Bayesian inference where the signal recovery can be enhanced by exploiting contextual and prior information. As suggested by \cite{Cevher_LearningCompressiblePriors_NIPS2009, Cevher_SparseRecoveryGraphModels_SPM2010}, sparsity can be induced via solving the following optimization problem:
\bea
    \max_{\vect x} P_{\vect x}(\vect x) & \textit{subject to} & ||\vect y - \mat A \vect x||_2 < \epsilon.        \label{Eq:GeneralMaxPrior}
\eea
where $P_{\vect x}$ is the probability distribution function of $\vect x$ that  captures  sparsity. The most common example is the i.i.d. Laplacian prior which is equivalent to $\ell_1$ norm minimization \cite{Cevher:SparseRecovGraphicalModel_SPMagaz2010}.
In this work, we focus on the spike and slab prior for inducing sparsity on $\vect x$. Using this prior, every coefficient $x_i$ is modeled as a mixture of two densities as follows:
\bea
    x_i \sim (1-\gamma_i) \delta_0 + \gamma_i P_i(x_i) \label{Eq:GeneralSpikeSlab}
\eea
where $\delta_0$ is the Dirac function at zero (spike) and $P_i$ (slab) is an appropriate prior distribution for nonzero values of $x_i$ (e.g. Gaussian). $\gamma_i \in [0,1]$ controls the structural sparsity of the signal. If $\gamma_i$ is chosen to be close to zero $x_i$  tends to remain zero. On the contrary, by choosing $\gamma_i$ close to 1, $P_i$ will be the dominant distribution encouraging $x_i$ to take a non-zero value.

\noindent\textbf{Optimization Problem} (Hierarchical Bayesian Framework):
Inspired by Bayesian compressive sensing (CS) \cite{JiAndCarin:BayesianCS_TSP2008, Suo:HierarchySpikeSlab_ICASSP2013}, we employ a hierarchical Bayesian framework for signal recovery.
In this model, priors are employed on $\vect y$ and $\vect x$. We also define $\gamma_i$ to be the indicator variable for the coefficient $x_i$, i.e. $\gamma_i \triangleq \mathbb{I}(x_i \neq 0)$. It takes the value one only if the corresponding coefficient $x_i$ is not zero, and zero otherwise. More precisely, the Bayesian formulation is as follows:
\bea
    \vect y | \mat A, \vect x, \vect \gamma, \sigma^2 & \sim &   \mathcal{N} \left(\mat A \vect x, \sigma^2\mat I \right) \label{eq:pdf_y}\\
    \vect x | \vect\gamma, \lambda, \sigma^2 & \sim &   \prod_{i=1}^{p} ~\gamma_{i} \mathcal{N}(0,\sigma^2\lambda^{-1}) + (1-\gamma_{i}) \delta_0 \label{eq:pdf_x}\\
    \vect\gamma | \vect  \kappa & \sim &   \prod_{i=1}^{p} ~\mbox{Bernoulli}(\kappa_{i}) \label{Eq:pdf_gamma}
\eea
where $\mathcal{N}(.)$ represents the Gaussian distribution. Also note that in \eqref{eq:pdf_x} each coefficient of $\vect x$ is modeled as i.i.d spike and slab prior. In addition,  a Bernoulli distribution is used to model the indicator variable $\gamma_i$ with parameter $\kappa_i$, which controls the sparseness of the signal. Motivated by a recent \emph{maximum a posteriori} (MAP) estimation technique proposed in \cite{Yen:MM_VariableSelectionSpikeSlab_Stat2011} the optimal $\vect x, \vect \gamma$ are obtained by the following MAP estimate.
\bea
        (\vect x^\ast, \vect \gamma^\ast )& =& \arg \max_{\vect x,\vect \gamma } \left\{  f (\vect x,\vect \gamma |\mat A,\vect y,\mat \kappa,\lambda ,\sigma^2)\right\}.
        \label{Eq:MAP_Estimate}
\eea
\begin{prop}  
The MAP estimation above is equivalent to the following minimization problem: 
\bea
    (\vect x^\ast, \vect \gamma^\ast ) &=& \arg\min_{\vect x, \vect \gamma} ~~   ||\vect y  - \mat A\vect x ||_{2}^2 + \lambda ||\vect x||_{2}^2 + \sum_{i=1}^{p}  \rho_{i}  \gamma_{i}       \label{Eq:MainOptProb}
\eea
where $\rho_{i} \triangleq  \sigma^2\log\left(\frac{2\pi\sigma^2(1-\kappa_i)^2}{\lambda\kappa_i^2}\right)$.
\end{prop}
\begin{proof}
See supplementary material. \footnote{Also available at { \url{http://signal.ee.psu.edu/ICR/ICRpage.htm} } }
\end{proof}
\emph{Remark:} The optimization problem in \eqref{Eq:MainOptProb} is a non-convex mixed integer programming involving the binary indicator variable $\vect \gamma$ and is not easily solvable using conventional optimization algorithms. It is worth mentioning that this is a more general formulation than the framework proposed in \cite{Srinivas:SSPIC_ICIP2013} or \cite{Yen:MM_VariableSelectionSpikeSlab_Stat2011} where authors simplified the optimization problem by assuming the same $\kappa$ for each coefficient $x_i$. This assumption changes the last term in \eqref{Eq:MainOptProb} to $\rho||\vect x||_0$ and the resulting optimization is solved in \cite{Yen:MM_VariableSelectionSpikeSlab_Stat2011} by using  Majorization-Minimization Methods. Further, a relaxation of $\ell_0$ to $\ell_1$ norm reduces the problem to the well-known Elastic-Net \cite{Zou:AdapElasticNet_AnnalStat2009}. The framework in \eqref{Eq:MainOptProb} therefore offers greater generality in capturing the sparsity of $\vect x$.  As an example, consider the scenario in a reconstruction or classification problem where some dictionary (training) columns are more important than others\cite{Mohammadi:PCADicLearningSRC_Elsevier2014}. It is then possible to encourage their contribution to the linear model by assigning higher values to the corresponding $\kappa_i$'s, which in turn makes it more likely that the $i^{th}$ coefficient $x_{i}$ becomes activated.
\section{Iterative Convex Refinement (ICR)}
We first develop a solution to \eqref{Eq:MainOptProb} for the case when the entries of $\vect x$ are non-negative. Then, we propose our method in its general form with no constraints.

The central idea of the proposed Iterative Convex Refinement (ICR) algorithm -- see Algorithm 1 -- is to generate a sequence of optimization problems that refines the solution of previous iteration based on solving a modified convex problem. At  iteration $n$ of ICR, the indicator variable $\gamma_i$ is replaced with the normalized ratio $\frac{x_i}{\mu_i^{(n-1)}}$ and the convex optimization problem in \eqref{Eq:NonNegOptProb} is solved which is a simple quadratic programming with non-negativity constraint. Note that, $\mu_i^{(n-1)}$ is intuitively the average value of optimal $x_i^{\ast}$'s obtained  from iteration $1$ up to $n-1$ and is rigorously defined as in \eqref{Eq:UpdateMu}.
The motivation for this substitution is that, if the sequence of solutions $\vect x^{(n)} $ converges to a point in $\mathbb{R}^p$ we also expect $\frac{x_i}{\mu_i^{(n-1)}}$  to converge to $\gamma_i$. Essentially, ICR is solving a sequence of convex quadratic programming problem that their solution converges to a sub-optimal solution of \eqref{Eq:MainOptProb}.

To generalize ICR to the unconstrained case, a simple modification is needed at each iteration. In fact, at each iteration \eqref{Eq:UnconsOptProb} is solved instead of \eqref{Eq:NonNegOptProb}. Note that \eqref{Eq:UnconsOptProb} is still convex and we solve it by alternating direction method of multipliers \cite{Boyd:ADMM_MachineLearn2011}. Again we expect the ratio $\frac{|x_i|}{ |\mu_i^{(n-1)} |}$ to converge to the value of optimal $\gamma_i$ and the result of ICR be a sub-optimal solution for \eqref{Eq:MainOptProb}. ICR in both its versions is summarized in Algorithm 1\footnote{The Matlab code for ICR is made available online at {\url{http://signal.ee.psu.edu/ICR/ICRpage.htm}}}.
\begin{algorithm}[t]
\caption{Iterative Convex Refinement (ICR)  }
\label{Alg:NonNeg}
\begin{algorithmic}
\REQUIRE $\mat A, \vect \kappa,  \vect y $.\\
\emph{initialize: } $\vect \mu^{(0)} = \mat A^T \vect y $, iteration index $n=1$.
\WHILE{Stopping criterion not met }
\STATE(1) Solve the convex optimization problem at iteration $n$:
\STATE(Non-negative) For non-negative ICR solve
\bea
	\vect x^{(n)}  =   \arg\min_{\vect x  \succcurlyeq \vect 0} ||\vect y - \mat A\vect x ||_2^2 + \lambda ||\vect x||_2^2 +  \sum_{i=1}^{p}  \rho_{i}   \frac{x_{i}}{\mu_{i}^{(n-1)}} \label{Eq:NonNegOptProb}
\eea
\STATE(Unconstrained) For unconstrained ICR solve
\bea
	\vect x^{(n)}  =   \arg\min_{\vect x} ||\vect y - \mat A\vect x ||_2^2 + \lambda ||\vect x||_2^2 +  \sum_{i=1}^{p}  \rho_{i}   \frac{|x_{i}|}{\big|\mu_{i}^{(n-1)}\big|} \label{Eq:UnconsOptProb}
\eea
\STATE(2) \vspace{-0.32in}
\bea
    \text{Update~} \mu_i^{(n)}:~~~   \mu_{i}^{(n)} = \frac{1}{n} \sum_{k=1}^{n} x_{i}^{(k)} ~~~ i=1,...,p\label{Eq:UpdateMu}
\eea
\STATE(3) Increase iteration index $n$.
\ENDWHILE{ if $ ||\vect x^{(n)}-\vect x^{(n-1)}|| \le tol$}
\ENSURE $\vect x^\ast = ~\vect x^{(n-1)} ,~  \gamma_i ^\ast = \frac{x_i^\ast}{\mu_i^{(n-1)}}$ for all $i=1,...,p$.
\end{algorithmic}
\end{algorithm}

To analyze the convergence properties of ICR, we first define the function $f_n : \mathbb{R}^p \rightarrow \mathbb{R}$ as follows: 
\bea
	f_n(\vect x) =  \vect x^T(\mat A^T \mat A +\lambda \mat I)\vect x  -2\vect y^T \mat A  \vect x  			+ \sum_{i=1}^{p} \frac{\rho_i}{\big|\mu_i^{(n-1)}\big|} |x_i|		\label{Eq:f_n}
\eea
which is another form of the functions to be minimized at each iteration of ICR. With this definition and assuming $\alpha$ is a constant that $\alpha <\frac{1}{2(q+p)}$, we propose the following two lemmas with proofs in the supplementary material$^1$:
\begin{lemma}
If \big|$\mu_j^{(n_0)}\big|< \alpha \rho_j$, then $x_j^{(n_0+1)} = 0$. ($\gamma_j \thickapprox \frac{x_j}{\mu_j^{(n_0)}} = 0$)
\end{lemma}
This lemma also implies that if $\big|\mu_j^{(n)}\big|< \alpha \rho_j$ for some $n$, then $x_j$ will remain zero for all the following iterations.
\begin{lemma}
If $\big|\mu_j^{(n)}\big| \ge \alpha \rho_j$ for all $n \ge n_0$, then there exists $N_j\ge n_0$ such that for all $n>N_j$ we have
\bea
	\bigg |\frac{1}{\big | \mu_j^{(n+1)} \big |} - 	\frac{1}{\big |\mu_j^{(n)} \big |} \bigg | \le \frac{c}{n+1}
\eea
where $c$ is some positive constant.
\end{lemma}
Another interpretation of this lemma is that as the number of iterations grows, the cost functions at each iteration of ICR get closer to each other. In view of these two lemmas, we can show that the sequence of optimal cost function values obtained from ICR algorithm forms a Quasi-Cauchy sequence \cite{Burton:QuasiCauchy_AmericanMAth2010}. In other words, this is a sequence of bounded values that their difference at two consecutive iterations gets smaller.
\begin{theorem}
After a sufficiently large $n$, the sequence of optimal cost function values obtained from ICR forms a Quasi-Cauchy sequence. i.e. $a_n = f_n(\vect x^{(n)})$  is a Quasi-Cauchy sequence of numbers.
\bea
    \big| f_{n+1}(\vect x^{(n+1)}) - f_n(\vect x^{(n)}) \big| \le \frac{c'}{n}. \label{Eq:Theorem1}
\eea
\end{theorem}
\begin{proof}
We provide a sketch of the proof here, for more details please refer to the supplementary material$^1$.

Before proving the theorem, note that we can assume for a sufficiently large $N_0$, if $n\ge N_0$, then  $\big|\mu_j^{(n)}\big|$ is either always less that $\alpha\rho_j$ or always greater (details can be found in the supplementary material). We now proceed to prove the Theorem and show that for $n>N_0$, the sequence of $f_n(\vect x^{(n)})$ satisfies the following property (assuming $|x_i|<1$):
\bea
    &&\big| f_{n+1}(\vect x^{(n)}) - f_{n}(\vect x^{(n)})\big| = \Big| \sum_{i=1}^{p} \rho_i\Big( \frac{1}{\big| \mu_i^{(n)} \big|} - \frac{1}{\big| \mu_i^{(n-1)} \big|} \Big)|x_i| \Big| \nonumber\\
    &\le& \sum_{|\mu_i^{(n-1)}|<\alpha\rho_i} \rho_i\Bigg| \frac{1}{\big| \mu_i^{(n)} \big|} - \frac{1}{\big| \mu_i^{(n-1)} \big|} \Bigg||x_i| \nonumber\\
    &&~~~~~~~~~~~~~~~~~~~~~~~+ \sum_{|\mu_i^{(n-1)}|\ge\alpha\rho_i} \rho_i\Bigg| \frac{1}{\big| \mu_i^{(n)} \big|} - \frac{1}{\big| \mu_i^{(n-1)} \big|} \Bigg||x_i|\nonumber\\
    &\le &\sum_{|\mu_i^{(n-1)}|\ge\alpha\rho_i} \rho_i\frac{c}{n}|x_i| ~\le~ p \max{\{\rho_i\}} \frac{c}{n} ~\le~ \frac{c'}{n}. \label{Eq:Property1}
\eea
This property also holds for $\vect x^{(n+1)}$. Finally, We show that for $n>N_0$, $a_n = f_n(\vect x^{(n)})$ is Quasi-Cauchy.
Since the minimum value $f_{n+1}(\vect x^{(n+1)})$ is smaller than $f_{n+1}(\vect x^{(n)})$, we can write:
\beaa
    f_{n+1}(\vect x^{(n+1)}) - f_{n}(\vect x^{(n)}) \le f_{n+1}(\vect x^{(n)}) - f_{n}(\vect x^{(n)}) \le \frac{c'}{n},
\eeaa
where we used \eqref{Eq:Property1} for $n>N_0$. With the same reasoning for $n>N_0$ we have:
\beaa
    f_{n+1}(\vect x^{(n+1)}) - f_{n}(\vect x^{(n)}) \ge f_{n+1}(\vect x^{(n+1)}) - f_{n}(\vect x^{(n+1)}) \ge  -\frac{c'}{n}.
\eeaa
Combining these two inequalities results \eqref{Eq:Theorem1} for $n>N_0$.
\end{proof}
Combination of this theorem with a reasonable \emph{stopping criterion} guarantees the termination of the ICR algorithm. The stopping criteria used in this case is the norm of difference in the solutions $\vect x^{(n)}$ in consecutive iterations. At termination where the solution converges, the ratio $\frac{x_i}{\mu_i^{(n)}}$ will be zero for zero coefficients and approaches 1 for nonzero coefficients, which matches the value of $\gamma_i$ in both cases.
\section{Experimental Validation}
We now apply the ICR method to sparse signal recovery problem using spike and slab priors. Two experimental scenarios are considered: 1.) synthetic data and 2.) a real-world image recovery problem. In each case, comparisons are made against state of the art alternatives.

\noindent \emph{Synthetic data:}  We set up a typical experiment for sparse recovery as in \cite{Mohimani:fast_l_0_TSP2009, Yen:MM_VariableSelectionSpikeSlab_Stat2011} with a randomly generated Gaussian matrix $\mat A \in \mathbb{R}^{q\times p}$ and a sparse vector $\vect x_0 \in \mathbb{R}^p$. Based on $\mat A$ and $\vect x_0$, we form the observation vector $\vect y\in \mathbb{R}^q$ according to the additive noise model: $ \vect y = \mat A \vect x_0 + \vect n$ with $\sigma = 0.01$. The competitive state-of-the-art methods for spike and slab sparse recovery that we compare against are: (1) SpaRSA \cite{Wright:SpaRSA_TSP2009,SPARSA:Code_Online} which is a  powerful method to solve the problems of the form \eqref{Eq:MainOptProb}. 
(2) Majorization Minimization (MM) algorithm \cite{Yen:MM_VariableSelectionSpikeSlab_Stat2011} which aims to solve the spike and slab signal recovery problem through a majorization minimization approach. 
(3) Adaptive Elastic Net \cite{Zou:AdapElasticNet_AnnalStat2009,Timofte:AdapElasticNet_PatternRecog2014} based on a $\ell_1$ relaxation of cost function. Initialization for all methods is consistent as suggested in \cite{SPARSA:Code_Online}.

Table \ref{Tab:table1} reports the   experimental results for a small scale problem. We chose to first report results on a small scale problem in order to be able to use the IBM ILOG CPLEX optimizer \cite{CPLEX:IBM_Online} which is a very powerful optimization toolbox for solving many different optimization problems. It can also find the \emph{global} solution  to non-convex and mixed-integer programming problems. We used this feature of CPLEX  to compare ICR's solution with the global minimizer. For obtaining the results in Table \ref{Tab:table1}, we choose $p=64$, $q=32$ and the sparsity level of $\vect x_0$ is $10$. We generated $1000$ realizations of $\mat A, \vect x_0$ and $\vect n$ and recovered $\vect x$ using different methods. Two different methods are used for evaluation of different sparse recovery methods: First, we compare different methods in terms of cost function value averaged over realizations, which is a direct measure of the quality of the solution to \eqref{Eq:MainOptProb}. Second, we compare performance of different methods from the sparse recovery point of view,  and used the following figures of merit: mean square error (MSE) with respect to the global solution ($\vect x_g$) obtained by CPLEX optimizer, ``Support Match'' (SM) measure indicating how much the support of each solution matches to that of $\vect x_g$.

As can be seen from Table \ref{Tab:table1}, ICR outperforms the competing methods in many different aspects. In particular from the first row, we infer that ICR is a better solution to \eqref{Eq:MainOptProb} since it achieves a better minimum in average sense. Moreover, significantly higher support match (SM $= 97.25\%$ ) measure for ICR shows that ICR's solution shows much more agreement with the global solution. Finally, the ICR solution is also the closest to the global solution obtained from CPLEX optimizer in the sense of MSE (by more than one order of magnitude in comparison with competing solutions).
    \begin{table}[t]
      \centering
      \caption{Comparison of methods for $p=64$ and $q=32$. On Average, stopping criterion for ICR is achieved at iteration $n=11$.}
      \label{Tab:table1}
      \begin{tabular}{|l||c|c|c|c|c|}
      \hline
      Method                                &SpaRSA          & MM            & Elastic Net    & ICR            \\
      \hline\hline
      Avg $f(\vect x^*)$                    &1.8244E-2       &1.4721E-2      &3.1938E-2       &\textbf{1.3379E-2}            \\
      \hline
      MSE  vs. $\vect x_g$                  &9.2668E-4       &2.2290E-3      &2.9210E-4       &\textbf{5.7851E-5}           \\
      \hline
      SM            vs. $\vect x_g$ ($\%$)   &84.46           &83.12          &67.46           &\textbf{97.25}              \\
      \hline
      \end{tabular}
    \end{table}
    \begin{table}
      \centering
      \caption{Comparison of methods for $p=512$ and $q=128$. On Average, stopping criterion for ICR is achieved at iteration $n=15$.}
      \label{Tab:table2}
      \begin{tabular}{|l||c|c|c|c|c|}
      \hline
      Method                                &SpaRSA          & MM            & Elastic Net    & ICR         \\
      \hline\hline
      Avg $f(\vect x^*)$                    &4.7510E-2       &3.9953E-2      &4.7885E-2       &\textbf{3.9127E-2}            \\
      \hline
       MSE vs. $\vect x_0$                  &1.1740E-3       &2.5685E-3      &9.3378E-4       &\textbf{3.9277E-4}              \\
      \hline
      Sparsity Level                        &55.91           &64.18          &85.30           &\textbf{28.82}              \\
      \hline
      SM vs. $\vect x_0$ ($\%$)             &89.68           &82.41          &87.61           &\textbf{96.33}              \\
      \hline
      \end{tabular}
    \end{table}
\begin{figure}[t]
  \centering
  \includegraphics[width=0.42\textwidth]{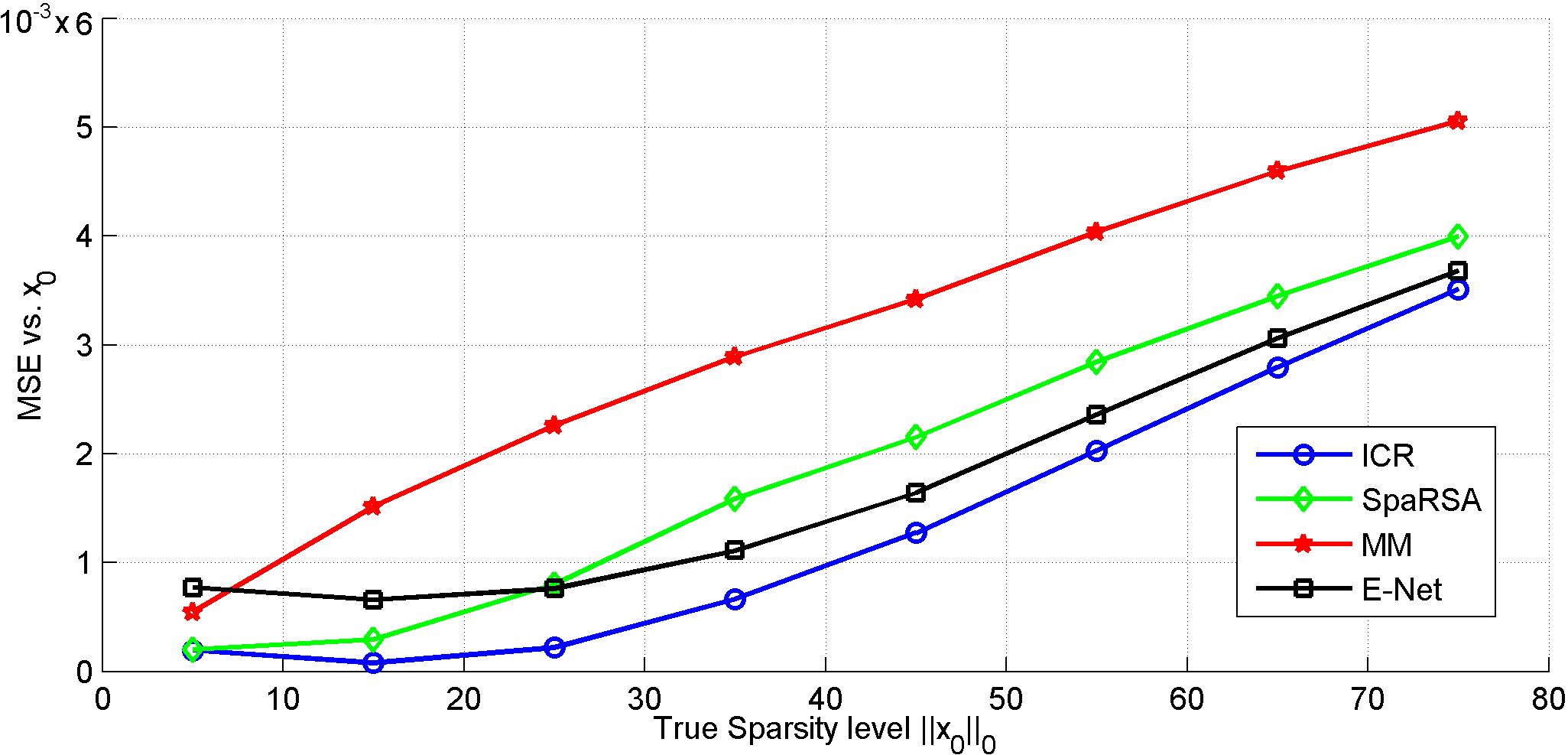}
  \caption{Comparison of average MSE each method versus sparsity level of $\vect x_0$.  }
  \label{Fig:MSEvsSPLevel}
\end{figure}

Next, we present results for a typical larger scale problem. We chose $p=512$, $q=128$ and set the sparsity level of $\vect x_0$ to be $30$ and carry out the same experiment as before. Because of the scale of the problem, the global solution is now unavailable and therefore, we compare the results against $\vect x_0$ which is the ``ground truth". Results are reported in Table \ref{Tab:table2}. Table \ref{Tab:table2} also additionally reports the average sparsity level of the solution  and it can be seen that the sparsity level of ICR is the closest to the true sparsity level of $\vect x_0$. In all other figures of merit, viz. the cost function value (averaged over realizations), MSE and support match vs. $\vect x_0$, ICR is again the best. Fig.\ \ref{Fig:MSEvsSPLevel} shows an alternate result as the MSE plotted against the sparsity level; once again the merits of ICR are readily apparent.

\noindent \emph{Image reconstruction:} In this part we aim to apply our ICR algorithm to real data for reconstruction of handwritten digit images from the well-known MNIST dataset \cite{Lecun:MNIST_Online}. The MNIST dataset contains 60000 digit images ($0$ to $9$) of size $28\times 28$ pixels. Most of pixels in these images are inactive and zero and only a few take non-zero values. Thus, these images are naturally sparse and fit into the spike and slab model. We set up this experiments such that a sparse signal $\vect x$ (vectorized image) is to be reconstructed from a smaller set of random measurements $\vect y$. For any particular image, we assume the smaller set of random measurement (150 measurements) is obtained by a Gaussian measurement matrix $\mat A \in \mathbb{R}^{150\times784}$ with added noise according to \eqref{Eq:y=Ax}. We compare our result against the following state-of-the-art image recovery methods for sparse images: 1.) SALSA-TV which uses the variable splitting proposed by Figueiredo \emph{et al.} \cite{MarioAndAfonso:ImageRecoverySALSA_TIP2010}  combined with Total Variation (TV) regularizers \cite{Chambolle:TV_Minimization_MathImagVision2004}. 2.) A Bayesian Image Reconstruction (BIR) \cite{Lu:SparseCodeBayesPerspec_NeuralNetLearn2013}, based on a more recent version of Bayesian image reconstruction method \cite{DobigeonAndHero:HierarchyBayesImageRecons_TIP2009} proposed by Hero \emph{et al.}. We also compare our results with Adaptive Elastic Net method \cite{Zou:AdapElasticNet_AnnalStat2009} which is commonly used in sparse image recovery problems. Finally, we also show the result of the non-negative version of ICR (ICR-NN) which explicitly enforces a non-negativity constraint on $\vect x$ which in this case corresponds to the intensity of reconstructed image pixels. Recovered images are shown in Fig. \ref{Fig:Digits} and the corresponding average reconstruction error (MSE) for the whole database for different methods appears next to each method. Clearly, ICR and ICR-NN outperform the other methods both visually and based on MSE value. It is also intuitively satisfying that ICR-NN which captures the non-negativity constraint natural to this problem, provides the best result overall.
\begin{figure}
\centering
    \includegraphics[width=0.42\textwidth]{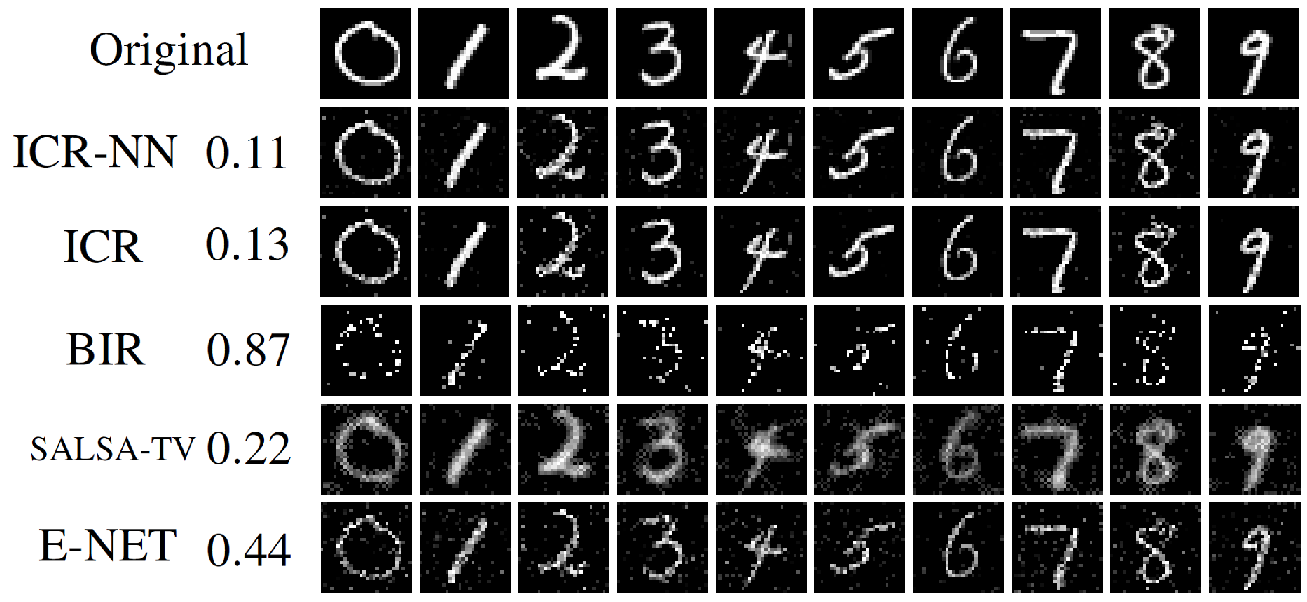}
    \caption{Examples of reconstructed images from MNIST dataset using different methods. The Numbers appeared next to each method is the average MSE for that method. On average, stoping criteria for ICR and ICR-NN are achieved at iteration $n=15$ and $n=29$, respectively.}
    \label{Fig:Digits}
\end{figure}

\section{Conclusion}
We develop a novel algorithm (ICR) for sparse recovery under spike and slab priors. Unlike known existing approaches, ICR does not simplify the optimization by assumptions/relaxations and hence affords a more general sparse structure. Experiments on synthetic data as well as a real-world image recovery problem confirms practical merits of ICR.  Future research may investigate further analysis of ICR properties and extensions to multi-task sparse recovery under collaborative spike and slab priors.

\setcounter{prop}{0}
\setcounter{lemma}{0}
\setcounter{theorem}{0}
\appendix[Experimental Validation]
In this Appendix, we show more experimental results from our framework to further support its significance in comparison with other state-of-the-art methods for spike and slab sparse recovery problem.
Following the same experimental setup for synthetic data as in the letter, we illustrate the performance of the ICR in comparison with others as the sparsity level of $\vect x_0$ ($||\vect x_0||_0$) changes. We vary the true sparsity level from only $5$ non-zero elements in $\vect x_0$ up to $95$ and compared MSE, support match percentage and the resulting sparsity level of the solutions from each method. Again we choose the length of sparse signal to be $p=512$ and number of observation to be $q=128$. Matrix $\mat A$ and sparse vector $\vect x_0$ are randomly generated and observation vector $\vect y$ is obtained by \eqref{Eq:y=Ax} with $\sigma = 0.01$. $1000$ realization of $\mat A$, $\vect x_0$ and $\vect n$ are generated for each sparsity level and the results are averaged over these $1000$ realizations. Figures \ref{Fig:MSEvsSPLevel}, \ref{Fig:SMvsSPLevel} and \ref{Fig:SpLevelvsSPLevel} illustrate these results.

Fig. \ref{Fig:SMvsSPLevel} illustrates that the support of ICR's solution is the closest to the support of $\vect x_0$. More than $90\%$ match between the support of ICR's solution  and that of $\vect x_0$ for a wide range of sparsity levels makes ICR very valuable to variable selection problems specially in Bayesian framework. Fig. \ref{Fig:SpLevelvsSPLevel} shows the actual sparsity level of solution for different methods. The dashed line corresponds to the true level of sparsity and ICR's solutions is the closest to the dashed line implying that the level of sparsity of ICR's solution matches the level of sparsity of $\vect x_0$ more than other methods. This also support the results obtained from Fig. \ref{Fig:SMvsSPLevel}.

\begin{figure}[h]
  \centering
  \includegraphics[width=0.48\textwidth]{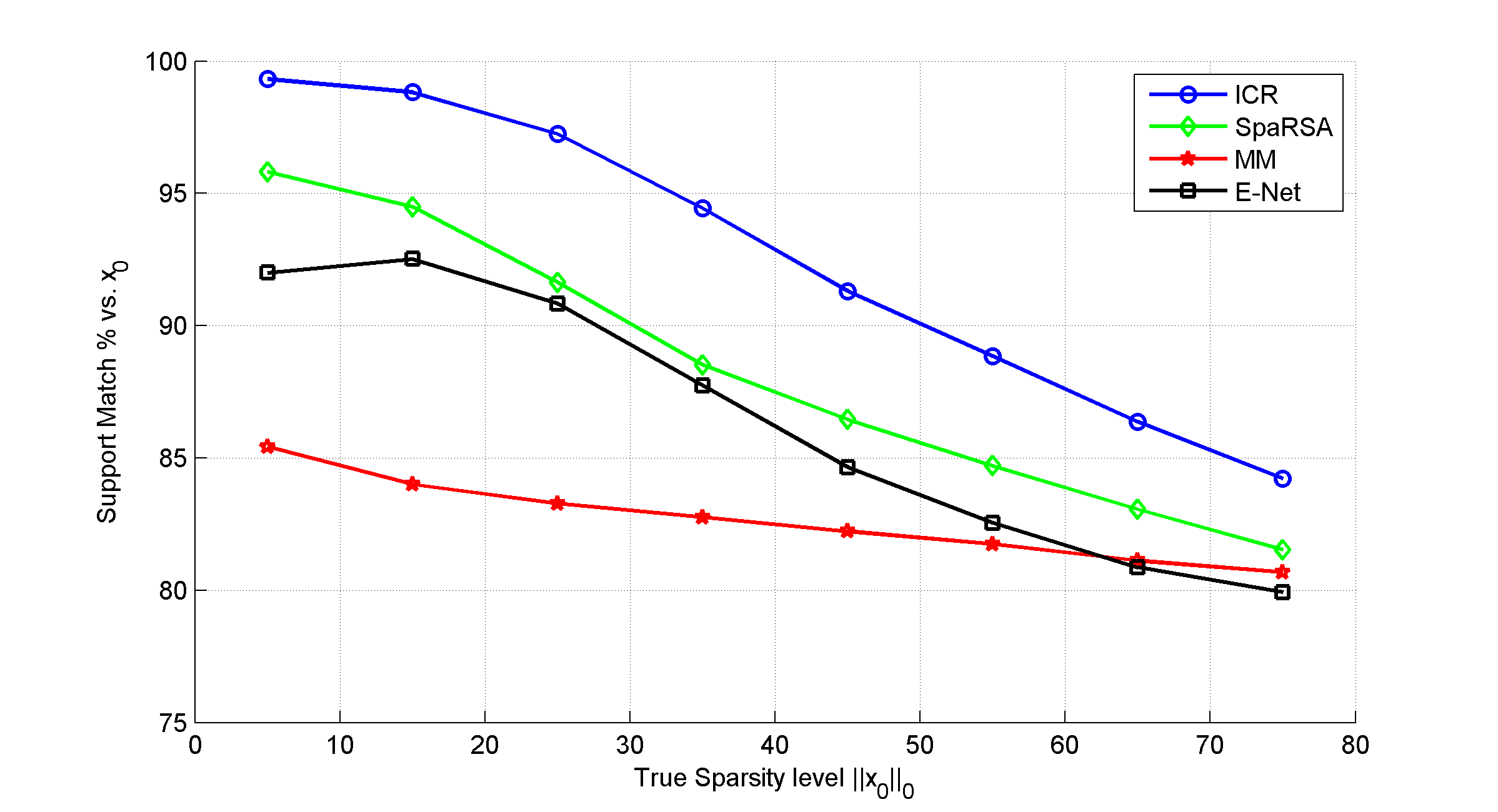}
  \caption{Comparison of average support match of the solution in $\%$ for each method versus sparsity level of $\vect x_0$.  }
  \label{Fig:SMvsSPLevel}
\end{figure}
\begin{figure}[h]
  \centering
  \includegraphics[width=0.48\textwidth]{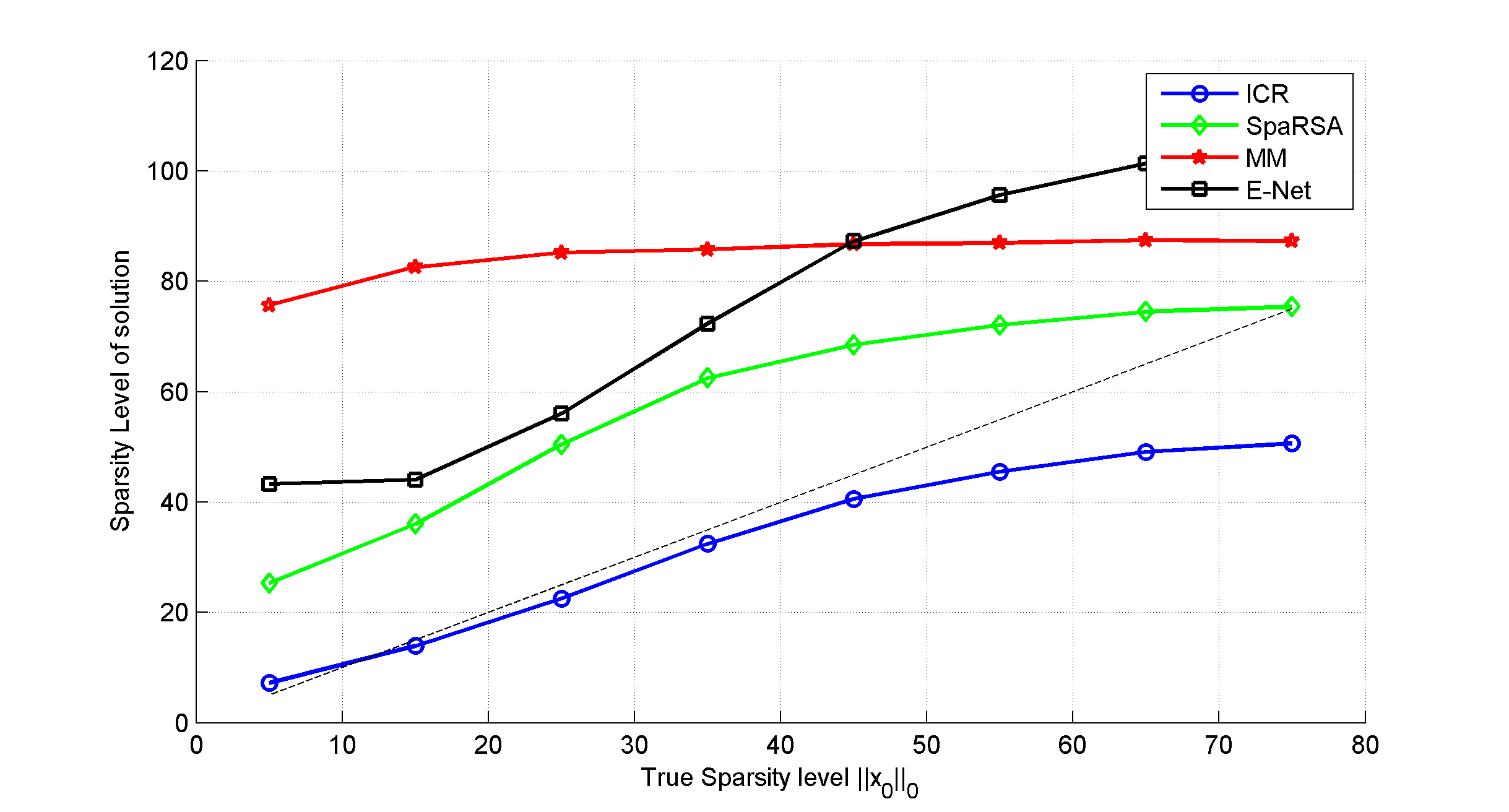}
  \caption{Comparison of average sparsity level obtained by each method versus sparsity level of $\vect x_0$. Dashed line shows the true level of sparsity  }
  \label{Fig:SpLevelvsSPLevel}
\end{figure}

\appendix[Analytical Results]
In this appendix, we present the proofs to the theoretical lemmas and theorems in the paper. For the rest of our analysis, without loss of generality we assume that $|y_i|\le 1$, $i=1...q$,   $|x_i|\le 1$, $i=1...p$ and columns of $\mat A$ have unity norm. We first begin with the proof to Proposition 1.
\begin{prop}
The MAP estimation in \eqref{Eq:MAP_Estimate} is equivalent to the following minimization problem:
\bea
    (\vect x^\ast, \vect \gamma^\ast ) &=& \arg\min_{\vect x, \vect \gamma} ~~   ||\vect y  - \mat A\vect x ||_{2}^2 + \lambda ||\vect x||_{2}^2 + \sum_{i=1}^{p}  \rho_{i}  \gamma_{i}
\eea
where $\rho_{i} \triangleq  \sigma^2\log\left(\frac{2\pi\sigma^2(1-\kappa_i)^2}{\lambda\kappa_i^2}\right)$.
\end{prop}

\begin{proof}
To perform the MAP estimation, note that the posterior probability is given by:
\be
f(\vect x,\vect \gamma, |\mat A,\vect y,\lambda, \vect \kappa) \propto f(\vect y |\mat A,\vect x,\vect \gamma,\sigma^2)f(\vect x|\vect \gamma,\sigma^2,\lambda) f(\vect \gamma | \vect \kappa).
\label{eq:joint-posterior}
\ee
The optimal $\vect x^\ast, \vect \gamma^\ast $ are obtained by MAP estimation as:
\be
(\vect x^\ast, \vect \gamma^\ast) = \arg \min_{\vect x,\vect \gamma} \left\{-2\log f(\vect x,\vect \gamma, |\mat A,\vect y,\lambda, \vect \kappa)\right\}.
\label{Eq:Log_MAP_Estimate}
\ee
We now separately evaluate each term on the right hand side of \eqref{eq:joint-posterior}. According to \eqref{eq:pdf_y} we have:
\beaa
f(\vect y |\mat A,\vect x,\vect \gamma,\sigma^2)  =    \frac{1}{(2\pi\sigma^2)^{q/2}} \exp\left\{-\frac{1}{2\sigma^2}(\vect y - \mat A \vect x)^T(\vect y - \mat A \vect x)\right\}
\eeaa
\beaa
\Rightarrow -2\log f(\vect y |\mat A,\vect x,\vect \gamma,\sigma^2)  =  q\log \sigma^2 + q\log (2\pi) + \frac{1}{\sigma^2}||\vect y - \mat A \vect x||^2.
\eeaa
Since $\gamma_i$ is assumed to be the indicator variable and only takes values $1$ and $0$, we can rewrite \eqref{eq:pdf_x} in the following form:
\beaa
     \vect x | \vect\gamma, \lambda, \sigma^2 & \sim &   \prod_{i=1}^{p} ~  \Big( \mathcal{N}(0,\sigma^2\lambda^{-1}) \Big) ^{\gamma_i} . \Big( \delta_0 \Big) ^{1-\gamma_i} \label{eq:pdf1_x}
\eeaa
Therefore
\beaa
& &f\big(\vect x |\vect \gamma,\sigma^2,\lambda\big) = \\
& &   \prod_{i=1}^{p} \left(\frac{1}{(2\pi\sigma^2/\lambda)^{1/2}}\right)^{\gamma_i} \exp\left(-\frac{\gamma_i x_i^2}{2\sigma^2\lambda^{-1}}\right)\delta_0^{1-\gamma_i}  \\
&=  &    \left(\frac{2\pi\sigma^2}{\lambda}\right)^{-\frac{1}{2}\sum_{i=1}^{p}\gamma_i} \exp\left\{-\frac{1}{2\sigma^2\lambda^{-1}}\sum_{i=1}^{p}\gamma_i x_i^2 \right\}   \prod_{i=1}^{p} \delta_0^{1-\gamma_i}\\
&= &\left(\frac{2\pi\sigma^2}{\lambda}\right) ^{-\frac{1}{2}  \sum_{i=1}^{p} \gamma_i} \exp\left(-\frac{  ||\vect x||_2^2 }{2\sigma^2 \lambda^{-1}} \right)   \prod_{i=1}^{p} \delta_0^{1-\gamma_i}
\eeaa
\beaa
&\Rightarrow & -2\log f(\vect x|\vect \gamma,\sigma^2,\lambda) = \\
& & \frac{  ||\vect x||_2^2}{\sigma^2\lambda^{-1}} + \log\left(\frac{2\pi\sigma^2}{\lambda}\right)  \sum_{i=1}^{p} \gamma_i -2 \sum_{i=1}^{p}(1-\gamma_i)\log\delta_0.
\eeaa
In fact $\delta_0 = \mathbb{I}(x_i = 0)$ and the final term on the right hand side evaluates to zero, since $\mathbb{I}(x_i = 0) = 1 \Rightarrow \log\mathbb{I}(x_i = 0) = 0$, and $\mathbb{I}(x_i = 0) = 0 \Rightarrow x_i \neq 0 \Rightarrow \gamma_i = 1 \Rightarrow (1-\gamma_i) = 0$.

Finally  \eqref{Eq:pdf_gamma} implies that
\beaa
f(\vect \gamma | \vect \kappa) & = &  \prod_{i=1}^{p}\kappa_i^{\gamma_i}(1-\kappa_i)^{1-\gamma_i} 
\eeaa
\beaa
\Rightarrow -2\log f(\vect \gamma | \vect \kappa) &=& -2 \sum_{i=1}^{p} \log \kappa_i^{\gamma_i} + \log (1-\kappa_i)^{1-\gamma_i} \\
&=& -2 \sum_{i=1}^{p} \gamma_i \log \kappa_i + (1-\gamma_i)\log (1-\kappa_i)\\
& = & -2 \sum_{i=1}^{p} \gamma_i \log \Big( \frac{\kappa_i}{1-\kappa_i}\Big) + \log (1-\kappa_i)\\
& = &  \sum_{i=1}^{p}\gamma_i\log\left(\frac{1-\kappa_i} {\kappa_i}\right)^2 - 2\sum_{i=1}^{p} \log(1-\kappa_i).
\eeaa
Plugging all these expressions back into  \eqref{Eq:Log_MAP_Estimate} and neglecting constant terms, we obtain:
\bea
(\vect x^\ast, \vect \gamma^\ast )  &=& \arg \min_{\vect x,\vect \gamma }
   q\log \sigma^2 + \frac{1}{\sigma^2}||\vect y - \mat A \vect x||^2 + \frac{  ||\vect x||_2^2}{\sigma^2\lambda^{-1}}  \nonumber \\
 &&  +  \log\left(\frac{2\pi\sigma^2}{\lambda}\right) \sum_{i=1}^{p} \gamma_i+  \sum_{i=1}^{p} \gamma_i\log\left(\frac{1-\kappa_i} {\kappa_i}\right)^2
\eea
Essentially, for fixed $\sigma^2$ The cost function will reduce to:
\bea
L(\vect x, \vect \gamma)& = &   ||\vect y - \mat A\vect x||_2^2 + \lambda ||\vect x||_2^2 +   \sum_{i=1}^{p}  \rho_ i \gamma_i
\eea
where $\rho_{i} \triangleq  \sigma^2\log\left(\frac{2\pi\sigma^2(1-\kappa_i)^2}{\lambda\kappa_i^2}\right)$.
\end{proof}

\begin{lemma}
If \big|$\mu_j^{(n_0)}\big|< \alpha \rho_j$, then $x_j^{(n_0+1)} = 0$. ($\gamma_j \thickapprox \frac{x_j}{\mu_j^{(n_0)}} = 0$)
\end{lemma}
\begin{proof}
Assume that for a specific $j$, $\big|\mu_j^{(n_0)}\big|< \alpha \rho_j$. Then for the next iteration the cost function to be minimized is as follows:
\bea
	f_{n_0+1}(\vect x) =  \vect x^T(\mat A^T \mat A +\lambda \mat I)\vect x  -2\vect y^T \mat A  \vect x  			+ \sum_{i=1}^{p} \frac{\rho_i}{\big|\mu_i^{(n_0)}\big|} |x_i|		\label{Eq:f_n+1}
\eea
Assume that the argument that minimizes \eqref{Eq:f_n+1} is $\vect x^{(n_0+1)}$. we can rewrite it in the following form:
\bea
    \vect x^{(n_0+1)} = \vect x_b + x_j \vect e_j
\eea
where $\vect e_j$ is the $j^{th}$ basis function with one at component $j$ and zeros elsewhere. $x_j$ is the $j^{th}$ element of $\vect x^{(n_0+1)}$ and $\vect x_b$ is equal to $\vect x^{(n_0+1)}$ except at $j^{th}$ element which is zero. We prove that if $\big|\mu_j^{(n_0)}\big|< \alpha \rho_j$, then $x_j = 0$.
\beaa
    f_{n_0+1}(\vect x_b)  &=& \vect x_b^T(\mat A^T \mat A +\lambda \mat I)\vect x_b  -2\vect y^T \mat A  \vect x_b  + \sum_{i=1}^{p} \frac{\rho_i}{\big|\mu_i^{(n_0)}\big|} |x_{b_i}|		\\
    f_{n_0+1}(\vect x^{(n_0+1)}) &=& ( \vect x_b+x_j \vect e_j )^T(\mat A^T \mat A +\lambda \mat I)( \vect x_b+x_j \vect e_j )   \\
    & &-2\vect y^T \mat A  ( \vect x_b+x_j \vect e_j ) + \sum_{i=1}^{p} \frac{\rho_i}{\big|\mu_i^{(n_0)}\big|} |x_{b_i}| \\
    & &+ \frac{\rho_j}{\big|\mu_j^{(n_0)}\big|} |x_j|
\eeaa
Therefore, their difference is:
\bea
    f_{n_0+1}\big (\vect x^{(n_0+1)}\big) - f_{n_0+1}\big(\vect x_b\big) ~ = ~x_j ^2 \vect e_j^T  (\mat A^T \mat A + \lambda \mat I) \vect e_j \nonumber\\
     ~~~~~+2 x_j \vect x_b^T (\mat A^T \mat A + \lambda \mat I) \vect e_j  - 2 x_j \vect y^T \mat A \vect e_j + \frac{\rho_j}{\big|\mu_j^{(n_0)}\big|} |x_j| \nonumber\\
    = \big |x_j \big| \Big ( |x_j| (\mat A^T \mat A +\lambda \mat I)_{jj} + \frac{\rho_j}{\big|\mu_j^{(n_0)}\big|}\Big ) \nonumber \\
    - 2 x_j \Big(  \vect y^T \mat A \vect e_j - \vect x_b^T(\mat A^T \mat A + \lambda \mat I)\vect e_j \Big) \label{Eq:difference}
\eea
We want to show that this difference is always positive except for $x_j=0$ which means $x_j$ must be zero in order for $f_{n_0+1}\big (\vect x^{(n_0+1)}\big)$ to be minimum. To do so, we show the following statements are true for nonzero $x_j$:
\beaa
    \Big| 2 x_j \Big(  \vect y^T \mat A \vect e_j - \vect x_b^T(\mat A^T \mat A + \lambda \mat I)\vect e_j \Big) \Big| <\big |x_j \big| \Big ( |x_j| (\mat A^T \mat A +\lambda \mat I)_{jj} + \frac{\rho_j}{\big|\mu_j^{(n_0)}\big|}\Big )
\eeaa
\bea
    \Leftrightarrow 2\Big|    \vect y^T \mat A \vect e_j - \vect x_b^T\mat A^T \mat A \vect e_j + \lambda \vect x_b^T \vect e_j \Big| &< &  |x_j| (\mat A^T \mat A +\lambda \mat I)_{jj} + \frac{\rho_j}{\big|\mu_j^{(n_0)}\big|} \nonumber\\
    \Leftrightarrow ~~~2 \Big| \vect y^T \mat A \vect e_j - \vect x_b^T\mat A^T \mat A\vect e_j \Big| &<&
    (1+\lambda)|x_j| + \frac{\rho_j}{\big|\mu_j^{(n_0)}\big|} \nonumber\\
    \Leftrightarrow ~~~~~~~2 \Big| (\vect y -\mat A \vect x_b)^T  \mat A \vect e_j \Big| &<&
    (1+\lambda)|x_j| + \frac{\rho_j}{\big|\mu_j^{(n_0)}\big|} \label{Eq:ineq}
\eea
In the above derivations, we used the fact that $(\mat A^T \mat A)_{jj} = 1$ since columns of $\mat A$ have unity norm. On the other hand, Cauchy-Schwarz inequality implies that,
\beaa
    2 \Big| (\vect y -\mat A \vect x_b)^T  \mat A \vect e_j \Big| &\le~~ 2||\vect y -\mat A \vect x_b||.||\mat A \vect e_j || &= ~~ 2||\vect y -\mat A \vect x_b|| \\
    &\le~~~ 2(||\vect y|| + ||\mat A \vect x_b||) &\le ~~2(\sqrt{q}+p)
\eeaa
Last inequality holds because of the fact that we assumed that magnitude of $x_i$ and $y_i$ do not exceed one. Also since we assumed $\big|\mu_j^{(n_0)}\big|< \alpha \rho_j$ and by definition of $\alpha$ we have:
\beaa
    (1+\lambda)|x_j| + \frac{\rho_j}{\big|\mu_j^{(n_0)}\big|} \ge \frac{1}{\alpha} \ge 2(q+p) \ge 2(\sqrt{q}+p)
\eeaa
Therefore, \eqref{Eq:ineq} is always true, since the right hand side is always greater than the left hand side. This implies that \eqref{Eq:difference} is positive for nonzero $x_j$ and, hence we must have $x_j$ = 0. Otherwise, it would contradict the fact that $f\big(\vect x^{(n_0+1)}\big)$ is the minimum value. Note that these are loose bounds and in practice they are easily satisfied. For example, $||\vect y - \mat A \vect x_b||$ is practically very small.
\end{proof}

\begin{lemma}
If $\big|\mu_j^{(n)}\big| \ge \alpha \rho_j$ for all $n \ge n_0$, then there exists $N_j\ge n_0$ such that for all $n>N_j$ we have
\bea
	\bigg |\frac{1}{\big | \mu_j^{(n+1)} \big |} - 	\frac{1}{\big |\mu_j^{(n)} \big |} \bigg | \le \frac{c}{n+1}
\eea
where $c$ is some positive constant.
\end{lemma}
\begin{proof}
%
%
Assume $\big|\mu_j^{(n)}\big| \ge \alpha\rho_j = \epsilon$.
First, note that it is straightforward to see that the difference of consecutive average values has the following property: $-\frac{1}{n+1}\le\big|\mu_j^{(n+1)}\big| - \big|\mu_j^{(n)}\big| \le \frac{1}{n+1}  $. Now, let $N_j = \frac{2}{\alpha\rho_j}$, then for all $n>N_j$ we have:
\bea
    \big|\mu_j^{(n)}\big| - \frac{1}{n+1} \le \big|\mu_j^{(n+1)}\big| \le \big|\mu_j^{(n)}\big| + \frac{1}{n+1}  \label{Eq:mu_jRange}
\eea
where the left hand side is  positive, since
\beaa
    \big|\mu_j^{(n)}\big| - \frac{1}{n+1} \ge \alpha \rho_j - \frac{1}{N_j} =\frac{\alpha \rho_j}{2} =\delta > 0 \label{eq:Temp}
\eeaa
Using this fact and \eqref{Eq:mu_jRange} we infer that:
\beaa
    &&\frac{1}{\big|\mu_j^{(n)}\big| + \frac{1}{n+1}}  \le \frac{1}{\big|\mu_j^{(n+1)}\big|} \le \frac{1}{\big|\mu_j^{(n)}\big| - \frac{1}{n+1}}\\
    &\Rightarrow  &\frac{1}{\big|\mu_j^{(n)}\big| + \frac{1}{n+1}} -\frac{1}{\big|\mu_j^{(n)}\big|} \le \frac{1}{\big|\mu_j^{(n+1)}\big|} -\frac{1}{\big|\mu_j^{(n)}\big|} \le \\
    && ~~~~~~~~~~~~~~~~~~~~~~~~~~~~~~~~~~~~~~~~\frac{1}{\big|\mu_j^{(n)}\big| - \frac{1}{n+1}} - \frac{1}{\big|\mu_j^{(n)}\big|}\\
    &\Rightarrow  &\frac{-\frac{1}{n+1}}{\Big(\big|\mu_j^{(n)}\big| + \frac{1}{n+1}\Big) \big|\mu_j^{(n)}\big| } \le \frac{1}{\big|\mu_j^{(n+1)}\big|} -\frac{1}{\big|\mu_j^{(n)}\big|} \le \\
    && ~~~~~~~~~~~~~~~~~~~~~~~~~~~~~~~~~~~~~~~~\frac{\frac{1}{n+1}}{\Big(\big|\mu_j^{(n)}\big| - \frac{1}{n+1}\Big)\big|\mu_j^{(n)} \big|}
\eeaa
In the last expression, we have:
\beaa
    \textit{RHS} = \frac{1}{(n+1)\Big(\big|\mu_j^{(n)}\big| - \frac{1}{n+1}\Big)\big|\mu_j^{(n)} \big|} \le \frac{1}{(n+1)\epsilon \delta}\\
    \textit{LHS} = \frac{-1}{(n+1)\Big(\big|\mu_j^{(n)}\big| + \frac{1}{n+1}\Big) \big|\mu_j^{(n)}\big| } \ge \frac{-1}{(n+1)\epsilon \delta}
\eeaa
Therefore,
\beaa
    \bigg|\frac{1}{\big|\mu_j^{(n+1)}\big|} -\frac{1}{\big|\mu_j^{(n)}\big|} \bigg| \le \frac{1}{(n+1)\epsilon\delta}, ~~~~ n>N_j.
\eeaa
\end{proof}

\begin{theorem}
After a sufficiently large $n$, the sequence of optimal cost function values obtained from ICR forms a Quasi-Cauchy sequence. i.e. $a_n = f_n(\vect x^{(n)})$  is a Quasi-Cauchy sequence of numbers.
\bea
    \big| f_{n+1}(\vect x^{(n+1)}) - f_n(\vect x^{(n)}) \big| \le \frac{c'}{n}
\eea
\end{theorem}
\begin{proof}
Before proving the theorem, note that we can assume for a sufficiently large $N_0$, if $n\ge N_0$, then  $\big|\mu_j^{(n)}\big|$ is either always less than $\alpha\rho_j$ or always greater. Because according to Lemma 1, we know that if $\big|\mu_j^{(n)}\big|$ once becomes smaller than $\alpha\rho_j$ for some $n$, it will remain less than $\alpha\rho_j$ for all the following iterations. Therefore, let $n_j, ~j=1...p$ be the iteration index that for all $n>n_j$, $ \big|\mu_j^{(n)}\big| <\epsilon$. Note that some $n_j$'s may be equal to infinity which means they are never smaller than $\epsilon$. For those $j$ that $n_j=\infty$, let $N_j$ to be the same as $N_j$ defined in proof of Lemma 2.
With these definitions, we now proceed to prove the Theorem. We first show that for $n>N_0 = \max(\max_j{n_j},\max_j{N_j})$, the sequence of $f_n(\vect x^{(n)})$ has the following property:
\bea
    &&\big| f_{n+1}(\vect x^{(n)}) - f_{n}(\vect x^{(n)})\big| = \Big| \sum_{i=1}^{p} \rho_i\Big( \frac{1}{\big| \mu_i^{(n)} \big|} - \frac{1}{\big| \mu_i^{(n-1)} \big|} \Big)|x_i| \Big| \nonumber\\
    &\le& \sum_{|\mu_i^{(n-1)}|<\epsilon} \rho_i\Bigg| \frac{1}{\big| \mu_i^{(n)} \big|} - \frac{1}{\big| \mu_i^{(n-1)} \big|} \Bigg||x_i| \nonumber\\
    &&~~~~~~~~~~~~~~~~~~~~~~~+ \sum_{|\mu_i^{(n-1)}|\ge\epsilon} \rho_i\Bigg| \frac{1}{\big| \mu_i^{(n)} \big|} - \frac{1}{\big| \mu_i^{(n-1)} \big|} \Bigg||x_i|\nonumber\\
    &\le &\sum_{|\mu_i^{(n-1)}|\ge\epsilon} \rho_i\frac{c}{n}|x_i| ~\le~ p \max{\{\rho_i\}} \frac{c}{n} ~\le~ \frac{c'}{n}. \label{Eq:Prop1}
\eea
This property also holds for $\vect x^{(n+1)}$. Finally, We show that for $n>N_0$, $a_n = f_n(\vect x^{(n)})$ is Quasi-Cauchy.
Since the minimum value $f_{n+1}(\vect x^{(n+1)})$ is smaller than $f_{n+1}(\vect x^{(n)})$, we can write:
\beaa
    f_{n+1}(\vect x^{(n+1)}) - f_{n}(\vect x^{(n)}) \le f_{n+1}(\vect x^{(n)}) - f_{n}(\vect x^{(n)}) \le \frac{c'}{n}
\eeaa
where we used \eqref{Eq:Prop1} for $n>N_0$. With the same reasoning for $n>N_0$ we have:
\beaa
    f_{n+1}(\vect x^{(n+1)}) - f_{n}(\vect x^{(n)}) \ge f_{n+1}(\vect x^{(n+1)}) - f_{n}(\vect x^{(n+1)}) \ge  -\frac{c'}{n}
\eeaa
Therefore,
\beaa
    \big| f_{n+1}(\vect x^{(n+1)}) - f_{n}(\vect x^{(n)}) \big| \le \frac{c'}{n}
\eeaa
for $n>N_0$.
\end{proof}
\emph{Remark:} Despite the fact that analytical results show a decay of order $\frac{1}{n}$ in difference between consecutive optimal cost function values, ICR shows much faster convergence in practice.


%
%
%
%
%
%
%
\ifCLASSOPTIONcaptionsoff
  \newpage
\fi
  \newpage
\bibliographystyle{IEEEtran}
\bibliography{SPlettersVer3}

\begin{thebibliography}{10}
\providecommand{\url}[1]{#1}
\csname url@samestyle\endcsname
\providecommand{\newblock}{\relax}
\providecommand{\bibinfo}[2]{#2}
\providecommand{\BIBentrySTDinterwordspacing}{\spaceskip=0pt\relax}
\providecommand{\BIBentryALTinterwordstretchfactor}{4}
\providecommand{\BIBentryALTinterwordspacing}{\spaceskip=\fontdimen2\font plus
\BIBentryALTinterwordstretchfactor\fontdimen3\font minus
  \fontdimen4\font\relax}
\providecommand{\BIBforeignlanguage}[2]{{%
\expandafter\ifx\csname l@#1\endcsname\relax
\typeout{** WARNING: IEEEtran.bst: No hyphenation pattern has been}%
\typeout{** loaded for the language `#1'. Using the pattern for}%
\typeout{** the default language instead.}%
\else
\language=\csname l@#1\endcsname
\fi
#2}}
\providecommand{\BIBdecl}{\relax}
\BIBdecl

\bibitem{Wright:SRC_PAMI2009}
J.~Wright, A.~Y. Yang, A.~Ganesh, S.~S. Sastry, and Y.~Ma, ``Robust face
  recognition via sparse representation,'' \emph{IEEE Trans.\ on Pattern
  Analysis and Machine Int.}, vol.~31, no.~2, pp. 210--227, 2009.

\bibitem{Srinivas:SHIRC_TMI2014}
U.~Srinivas, H.~S. Mousavi, V.~Monga, A.~Hattel, and B.~Jayarao, ``Simultaneous
  sparsity model for histopathological image representation and
  classification.'' \emph{IEEE Trans.\ on Medical Imaging}, vol.~33, no.~5, pp.
  1163--1179, 2014.

\bibitem{Srinivas:SSPIC_ICIP2013}
U.~Srinivas, Y.~Suo, M.~Dao, V.~Monga, and T.~D. Tran, ``Structured sparse
  priors for image classification.'' in \emph{Proc.\ IEEE Conf.\ on Image
  Processing}, 2013, pp. 3211--3215.

\bibitem{Srinivas:SHIRC_ISBI2013}
U.~Srinivas, H.~S. Mousavi, C.~Jeon, V.~Monga, A.~Hattel, and B.~Jayarao,
  ``{SHIRC}: A simultaneous sparsity model for histopathological image
  representation and classification,'' in \emph{Proc.\ IEEE Int.\ Symp.\
  Biomed.\ Imag.}, 2013, pp. 1118--1121.

\bibitem{Mousavi:MICHS_ICIP2014}
H.~S. Mousavi, U.~Srinivas, V.~Monga, Y.~Suo, M.~Dao, and T.~D. Tran,
  ``Multi-task image classification via collaborative, hierarchical
  spike-and-slab priors,'' in \emph{Proc.\ IEEE Conf.\ on Image Processing},
  2014, pp. 4236--4240.

\bibitem{Bahrampour:TreeSparsity_CVPR2014}
S.~Bahrampour, A.~Ray, N.~M. Nasrabadi, and K.~W. Jenkins, ``Quality-based
  multimodal classification using tree-structured sparsity,'' in \emph{Proc.\
  IEEE Conf.\ Computer Vision Pattern Recognition}.\hskip 1em plus 0.5em minus
  0.4em\relax IEEE, 2014, pp. 4114--4121.

\bibitem{Suo1:DirtyDicLearn_ICIP2014}
Y.~Suo, M.~Dao, T.~Tran, H.~Mousavi, U.~Srinivas, and V.~Monga, ``Group
  structured dirty dictionary learning for classification,'' in \emph{Proc.\
  IEEE Conf.\ on Image Processing}, 2014, pp. 150--154.

\bibitem{Pourkamali:CompresiveKSVD_ICASSP2013}
F.~P. Anaraki and S.~M. Hughes, ``Compressive k-svd,'' in \emph{Proc.\ IEEE
  Int.\ on Conf.\ Acoustics, Speech, and Signal Processing}.\hskip 1em plus
  0.5em minus 0.4em\relax IEEE, 2013, pp. 5469--5473.

\bibitem{SadeghiAndBabaiezade1:DicLearnSparse_SPLetter2013}
M.~Sadeghi, M.~Babaie-Zadeh, and C.~Jutten, ``Dictionary learning for sparse
  representation: A novel approach,'' \emph{IEEE Signal Processing Letters},
  vol.~20, no.~12, pp. 1195--1198, Dec 2013.

\bibitem{Vu:DFDL_ISBI2015}
T.~H. Vu, H.~S. Mousavi, V.~Monga, U.~Rao, and G.~Rao, ``{DFDL}: Discriminative
  feature-oriented dictionary learning for histopathological image
  classification,'' \emph{arXiv preprint arXiv:1502.01032}, 2015.

\bibitem{Bahrampour:DicLearn_Arxiv2015}
S.~Bahrampour, N.~M. Nasrabadi, A.~Ray, and W.~K. Jenkins, ``Multimodal
  task-driven dictionary learning for image classification,'' \emph{arXiv
  preprint arXiv:1502.01094}, 2015.

\bibitem{Bahrampour:KernelDicLearn_Arxiv2015}
------, ``Kernel task-driven dictionary learning for hyperspectral image
  classification,'' \emph{arXiv preprint arXiv:1502.03126}, 2015.

\bibitem{Wright:SpaRSA_TSP2009}
S.~J. Wright, R.~D. Nowak, and M.~A. Figueiredo, ``Sparse reconstruction by
  separable approximation,'' \emph{IEEE Trans.\ on Signal Processing}, vol.~57,
  no.~7, pp. 2479--2493, 2009.

\bibitem{Tropp:OMP_InfoTheory2007}
J.~A. Tropp and A.~C. Gilbert, ``Signal recovery from random measurements via
  orthogonal matching pursuit,'' \emph{IEEE Trans.\ on Info.\ Theory}, vol.~53,
  no.~12, pp. 4655--4666, 2007.

\bibitem{Elad:ImageDenoiseSparsity_TIP2006}
M.~Elad and M.~Aharon, ``Image denoising via sparse and redundant
  representations over learned dictionaries,'' \emph{IEEE Trans.\ on Image
  Processing}, vol.~15, no.~12, pp. 3736--3745, 2006.

\bibitem{YangAndWright:SparseSR_TIP2010}
J.~Yang, J.~Wright, T.~S. Huang, and Y.~Ma, ``Image super-resolution via sparse
  representation,'' \emph{IEEE Trans.\ on Image Processing}, vol.~19, no.~11,
  pp. 2861--2873, 2010.

\bibitem{Andersen:BayesianSpikeSlab_NIPS2014}
M.~R. Andersen, O.~Winther, and L.~K. Hansen, ``Bayesian inference for
  structured spike and slab priors,'' in \emph{Advances in Neural Information
  Processing Systems}, 2014, pp. 1745--1753.

\bibitem{Sprechmann:CHI-LASSO_TSP2011}
P.~Sprechmann, I.~Ramirez, G.~Sapiro, and Y.~C. Eldar, ``C-hilasso: A
  collaborative hierarchical sparse modeling framework,'' \emph{IEEE Trans.\ on
  Signal Processing}, vol.~59, no.~9, pp. 4183--4198, 2011.

\bibitem{Baraniuk:Model_CS_InfoTheory2010}
R.~G. Baraniuk, V.~Cevher, M.~F. Duarte, and C.~Hegde, ``Model-based
  compressive sensing,'' \emph{IEEE Trans.\ on Image Processing}, vol.~56,
  no.~4, pp. 1982--2001, 2010.

\bibitem{Carin:WaveletBayesCS_TSP2009}
L.~He and L.~Carin, ``Exploiting structure in wavelet-based bayesian
  compressive sensing,'' \emph{IEEE Trans.\ on Signal Processing}, vol.~57,
  no.~9, pp. 3488--3497, 2009.

\bibitem{JiAndCarin:BayesianCS_TSP2008}
S.~Ji, Y.~Xue, and L.~Carin, ``Bayesian compressive sensing,'' \emph{IEEE
  Trans.\ on Signal Processing}, vol.~56, no.~6, pp. 2346--2356, 2008.

\bibitem{Cai:OMP_InfoTheory2011}
T.~T. Cai and L.~Wang, ``Orthogonal matching pursuit for sparse signal recovery
  with noise,'' \emph{IEEE Trans.\ on Info.\ Theory}, vol.~57, no.~7, pp.
  4680--4688, 2011.

\bibitem{Mousavi:AssymLASSO_arXive2013}
A.~Mousavi, A.~Maleki, and R.~G. Baraniuk, ``Asymptotic analysis of lassos
  solution path with implications for approximate message passing,''
  \emph{arXiv preprint arXiv:1309.5979}, 2013.

\bibitem{Mohimani:fast_l_0_TSP2009}
H.~Mohimani, M.~Babaie-Zadeh, and C.~Jutten, ``A fast approach for overcomplete
  sparse decomposition based on smoothed norm,'' \emph{IEEE Trans.\ on Signal
  Processing}, vol.~57, no.~1, pp. 289--301, 2009.

\bibitem{Lu:SparseCodeBayesPerspec_NeuralNetLearn2013}
X.~Lu, Y.~Wang, and Y.~Yuan, ``Sparse coding from a bayesian perspective,''
  \emph{Neural Networks and Learning Systems, IEEE Transactions on}, vol.~24,
  no.~6, pp. 929--939, 2013.

\bibitem{DobigeonAndHero:HierarchyBayesImageRecons_TIP2009}
N.~Dobigeon, A.~O. Hero, and J.-Y. Tourneret, ``Hierarchical bayesian sparse
  image reconstruction with application to mrfm,'' \emph{IEEE Trans.\ on Image
  Processing}, vol.~18, no.~9, pp. 2059--2070, 2009.

\bibitem{BeckerAndCandes:SparseRecoveryNESTA_ImagScienSIAM2011}
S.~Becker, J.~Bobin, and E.~J. Cand{\`e}s, ``Nesta: a fast and accurate
  first-order method for sparse recovery,'' \emph{SIAM Journal on Imaging
  Sciences}, vol.~4, no.~1, pp. 1--39, 2011.

\bibitem{Beck:IterativeShrinkageThresholdFISTA_ImagScienSIAM2009}
A.~Beck and M.~Teboulle, ``A fast iterative shrinkage-thresholding algorithm
  for linear inverse problems,'' \emph{SIAM Journal on Imaging Sciences},
  vol.~2, no.~1, pp. 183--202, 2009.

\bibitem{Boyd:ADMM_MachineLearn2011}
S.~Boyd, N.~Parikh, E.~Chu, B.~Peleato, and J.~Eckstein, ``Distributed
  optimization and statistical learning via the alternating direction method of
  multipliers,'' \emph{Foundations and Trends{\textregistered} in Machine
  Learning}, vol.~3, no.~1, pp. 1--122, 2011.

\bibitem{Babacan_BayesianCSLaplacePriors_TIP2010}
S.~Babacan, R.~Molina, and A.~Katsaggelos, ``Bayesian compressive sensing using
  laplace priors,'' \emph{IEEE Trans.\ on Image Processing}, vol.~19, no.~1,
  pp. 53--63, 2010.

\bibitem{Cevher:SparseRecovGraphicalModel_SPMagaz2010}
V.~Cevher, P.~Indyk, L.~Carin, and R.~G. Baraniuk, ``Sparse signal recovery and
  acquisition with graphical models,'' \emph{Signal Processing Magazine, IEEE},
  vol.~27, no.~6, pp. 92--103, 2010.

\bibitem{Mitchell:BayesVarSelectSpikeSlab_StatAssoc1988}
T.~J. Mitchell and J.~J. Beauchamp, ``Bayesian variable selection in linear
  regression,'' \emph{Journal of the American Statistical Association},
  vol.~83, no. 404, pp. 1023--1032, 1988.

\bibitem{Ishwaran_SpikeSlab_AnnStat2005}
H.~Ishwaran and J.~S. Rao, ``Spike and slab variable selection: frequentist and
  bayesian strategies,'' \emph{Annals of Statistics}, pp. 730--773, 2005.

\bibitem{Suo:HierarchySpikeSlab_ICASSP2013}
Y.~Suo, M.~Dao, T.~Tran, U.~Srinivas, and V.~Monga, ``Hierarchical sparse
  modeling using spike and slab priors,'' in \emph{Proc.\ IEEE Int.\ on Conf.\
  Acoustics, Speech, and Signal Processing}.\hskip 1em plus 0.5em minus
  0.4em\relax IEEE, 2013, pp. 3103--3107.

\bibitem{Lazaro:SpikeSlabInferMultiTask_NIPS2011}
M.~L{\'a}zaro-gredilla and M.~K. Titsias, ``Spike and slab variational
  inference for multi-task and multiple kernel learning,'' in \emph{Advances in
  neural information processing systems}, 2011, pp. 2339--2347.

\bibitem{Yen:MM_VariableSelectionSpikeSlab_Stat2011}
T.-J. Yen \emph{et~al.}, ``A majorization--minimization approach to variable
  selection using spike and slab priors,'' \emph{The Annals of Statistics},
  vol.~39, no.~3, pp. 1748--1775, 2011.

\bibitem{Cevher_LearningCompressiblePriors_NIPS2009}
V.~Cevher, ``Learning with compressible priors,'' in \emph{Advances in Neural
  Information Processing Systems}, 2009, pp. 261--269.

\bibitem{Cevher_SparseRecoveryGraphModels_SPM2010}
V.~Cevher, P.~Indyk, L.~Carin, and R.~G. Baraniuk, ``Sparse signal recovery and
  acquisition with graphical models,'' \emph{Signal Processing Magazine, IEEE},
  vol.~27, no.~6, pp. 92--103, 2010.

\bibitem{Zou:AdapElasticNet_AnnalStat2009}
H.~Zou and H.~H. Zhang, ``On the adaptive elastic-net with a diverging number
  of parameters,'' \emph{Annals of statistics}, vol.~37, no.~4, p. 1733, 2009.

\bibitem{Mohammadi:PCADicLearningSRC_Elsevier2014}
M.~Mohammadi, E.~Fatemizadeh, and M.~Mahoor, ``Pca-based dictionary building
  for accurate facial expression recognition via sparse representation,''
  \emph{Journal of Visual Communication and Image Representation}, vol.~25,
  no.~5, pp. 1082--1092, 2014.

\bibitem{Burton:QuasiCauchy_AmericanMAth2010}
D.~Burton and J.~Coleman, ``Quasi-cauchy sequences,'' \emph{The American
  Mathematical Monthly}, vol. 117, no.~4, pp. 328--333, 2010.

\bibitem{SPARSA:Code_Online}
\BIBentryALTinterwordspacing
S.~J. Wright, R.~D. Nowak, and M.~Figueiredo. (2014, Jul.) {SpaRSA} software.
  [Online]. Available: \url{http://www.lx.it.pt/~mtf/SpaRSA/}
\BIBentrySTDinterwordspacing

\bibitem{Timofte:AdapElasticNet_PatternRecog2014}
R.~Timofte and L.~Van~Gool, ``Adaptive and weighted collaborative
  representations for image classification,'' \emph{Pattern Recognition
  Letters}, vol.~43, pp. 127--135, 2014.

\bibitem{CPLEX:IBM_Online}
\BIBentryALTinterwordspacing
{IBM}. (2014, Oct.) {ILOG} {CPLEX} optimization studio. [Online]. Available:
  \url{http://www-01.ibm.com/software/commerce/optimization/cplex-optimizer/}
\BIBentrySTDinterwordspacing

\bibitem{Lecun:MNIST_Online}
\BIBentryALTinterwordspacing
L.~Yann, C.~Cortes, and C.~J. Burges. (2014, Dec.) {MNIST} dataset. [Online].
  Available: \url{http://yann.lecun.com/exdb/mnist/}
\BIBentrySTDinterwordspacing

\bibitem{MarioAndAfonso:ImageRecoverySALSA_TIP2010}
M.~V. Afonso, J.~M. Bioucas-Dias, and M.~A. Figueiredo, ``Fast image recovery
  using variable splitting and constrained optimization,'' \emph{IEEE Trans.\
  on Image Processing}, vol.~19, no.~9, pp. 2345--2356, 2010.

\bibitem{Chambolle:TV_Minimization_MathImagVision2004}
A.~Chambolle, ``An algorithm for total variation minimization and
  applications,'' \emph{Journal of Mathematical imaging and vision}, vol.~20,
  no. 1-2, pp. 89--97, 2004.

\end{thebibliography}

\end{document}